\documentclass[11pt]{article}
\usepackage[a4paper,margin=1in]{geometry}
\usepackage[T1]{fontenc}
\usepackage{newtxtext}
\usepackage{amsmath,amssymb,amsthm}
\usepackage{newtxmath}
\usepackage{microtype}
\usepackage{mathtools}
\usepackage{booktabs}
\usepackage{multirow}
\usepackage{graphicx}
\usepackage{caption}
\usepackage{subcaption}
\usepackage{enumitem}
\usepackage{xcolor}
\usepackage{float}
\usepackage[section]{placeins}
\usepackage{algorithm}
\usepackage{algpseudocode}
\usepackage{bm}
\usepackage{hyperref}
\usepackage{tikz}
\usepackage{tikz-3dplot}
\usetikzlibrary{shapes,arrows,calc}

\hypersetup{colorlinks=true,linkcolor=blue,citecolor=blue,urlcolor=blue}

\usepackage[capitalize,noabbrev]{cleveref}

\newtheorem{theorem}{Theorem}[section]
\newtheorem{proposition}[theorem]{Proposition}
\newtheorem{lemma}[theorem]{Lemma}
\newtheorem{corollary}[theorem]{Corollary}
\theoremstyle{definition}
\newtheorem{definition}[theorem]{Definition}
\theoremstyle{remark}
\newtheorem{remark}[theorem]{Remark}

\newcommand{\E}{\mathbb{E}}

\DeclareMathOperator*{\argmin}{arg\,min}
\DeclareMathOperator{\sg}{sg}

\newcommand{\loss}{\mathcal{L}}

\newcommand{\piref}{\pi_{\mathrm{ref}}}
\newcommand{\pik}{\pi_k}

\title{\textbf{Orthogonalized Policy Optimization}\\
\large Policy Optimization as Orthogonal Projection in Hilbert Space}

\author{Wang Zixian\\
\small China Mobile Communications Group Shandong Co., Ltd. Tai'an Branch\\
\small \texttt{wangzixian@sd.chinamobile.com}}
\date{}

\begin{document}
\maketitle

\begin{abstract}
We propose \textbf{Orthogonalized Policy Optimization (OPO)}, a principled framework for large language model alignment derived from optimization in the Hilbert function space $L^2(\pi_k)$. Lifting policy updates from the probability simplex into $L^2(\pi_k)$ transforms the nonlinear normalization constraint into a linear orthogonality condition $\langle v, \mathbf{1} \rangle_{\pi_k} = 0$ on the density fluctuation field $v = \pi/\pi_k - 1$. By the Hilbert Projection Theorem, the unique closed-form update is 
$v^* = (\omega_\alpha - \mathbb{E}[\omega_\alpha])/\mu$, 
where the subtracted mean acts as a chemical potential enforcing probability conservation. This interpretation reveals advantage $z$-score normalization as a conservation-law projection rather than a variance-reduction heuristic.

OPO cleanly decouples \emph{sampling geometry}, controlled by the escort exponent $\alpha$, from \emph{optimization geometry}, governed by the stiffness parameter $\mu$, a separation not attainable under KL-based objectives. The same update emerges independently as a Euclidean mirror-descent step and as the linear-response law of near-equilibrium statistical mechanics, establishing its structural uniqueness within ratio geometry.

Structurally, OPO induces constant curvature, non-saturating linear gradient dynamics, and an intrinsic $\chi^2$ trust region. Experiments on MATH benchmarks show that the Hilbert projection formulation prevents the gradient saturation typical of KL-constrained methods. By sustaining non-vanishing gradients in high-confidence regimes, OPO avoids premature plateaus and achieves stronger long-horizon training rewards and improved out-of-distribution generalization compared to clipping-based baselines.
\end{abstract}

Aligning Large Language Models (LLMs) with human preferences has become a cornerstone of modern AI development. The standard paradigm, Reinforcement Learning from Human Feedback (RLHF)~\cite{Christiano2017RLHF,Ouyang2022InstructGPT}, typically employs algorithms like Proximal Policy Optimization (PPO)~\cite{Schulman2017PPO} or Direct Preference Optimization (DPO)~\cite{Rafailov2023DPO} to optimize a policy against a reward model or preference dataset. Despite their formulation differences, these methods share a common structural foundation: they all rely on Kullback--Leibler (KL) divergence to constrain the policy update, preventing it from deviating excessively from a reference distribution.

While empirically effective, KL-regularized objectives exhibit distinct limitations in reasoning-intensive domains where high confidence is required. As the policy improves and assigns high probability to correct reasoning chains, the KL penalty---which induces an exponential geometry in the log-probability space---often dominates the learning signal. This manifests as gradient saturation, where the driving force for further improvement vanishes exponentially as the model becomes confident. Consequently, training often plateaus prematurely, behavior commonly attributed to ``over-regularization'' but which we argue is intrinsic to the chosen geometry.

We contend that this issue stems from a structural conflation in existing objective designs. Fundamentally, alignment involves two independent design choices: \textit{sampling geometry}, which determines the effective weighting of training examples (e.g., whether to focus on high-advantage samples), and \textit{optimization geometry}, which determines the curvature of the update step (e.g., how to measure distance in the policy space). In standard KL-based methods, a single divergence term dictates both, coupling the exploration strength with the optimization stability. Adjusting one inevitably perturbs the other, creating a dilemma where aggressive sampling destabilizes training, while stable optimization stifles exploration.

\paragraph{A Hilbert Space Perspective.}
To resolve this, we propose reformulating alignment as a constrained optimization problem in the Hilbert function space $L^2(\pik)$. This shift yields three structural advantages:
\begin{enumerate}[leftmargin=2em]
    \item The probability conservation constraint $\sum_y \pi(y) = 1$ becomes a \emph{linear} orthogonality condition $\langle v, \mathbf{1}\rangle_{\pik} = 0$, defining a closed subspace $\mathcal{H}_0 \subset L^2(\pik)$.
    \item The optimal policy update is obtained directly by the \emph{Hilbert Projection Theorem}: the closest point in $\mathcal{H}_0$ to the advantage-driven target. The Lagrange multiplier (chemical potential) enforcing conservation emerges organically as the projected-out component.
    \item The Hessian of the resulting objective is the constant scalar $\mu I$, completely independent of the data distribution or the current policy state.
\end{enumerate}

We further show that this Hilbert-space derivation is equivalent to two other well-known mathematical structures---Bregman mirror descent with a Euclidean mirror map, and near-equilibrium linear response from statistical mechanics---confirming that \textbf{Orthogonalized Policy Optimization (OPO)} is the unique quadratic proximal response in ratio geometry.

Our contributions are as follows:
\begin{itemize}
    \item We identify the implicit coupling of sampling and optimization geometries in KL-based alignment methods as a root cause of gradient saturation.
    \item We formulate alignment as a constrained optimization in $L^2(\pik)$ and derive OPO via the Hilbert Projection Theorem, revealing that advantage $z$-score normalization is the unique conservation-law projection.
    \item We prove equivalence to mirror descent and linear response, confirming uniqueness of the quadratic proximal response in ratio geometry.
    \item We empirically demonstrate that OPO outperforms strong baselines (GRPO, GSPO, DAPO) on mathematical reasoning tasks across two experimental settings: long-horizon training dynamics (224 steps) and out-of-distribution generalization under data scarcity. A \emph{gradient response efficiency} analysis reveals that OPO maintains ${\sim}5{\times}$ higher gradient-per-signal than KL-based baselines.
\end{itemize}

\section{Related Work}
\label{sec:related}

\paragraph{Preference Optimization and RLHF.}
Reinforcement Learning from Human Feedback (RLHF) typically involves learning a reward model from preferences and then optimizing a policy via PPO~\cite{Schulman2017PPO,Christiano2017RLHF,Ouyang2022InstructGPT}. Direct Preference Optimization (DPO)~\cite{Rafailov2023DPO} simplifies this by deriving a closed-form solution to the KL-constrained reward maximization problem. Recent variants extend this paradigm: IPO~\cite{Azar2024IPO} adds a regularization term to prevent overfitting, SimPO~\cite{Meng2024SimPO} simplifies the reference-free objective.

\paragraph{$f$-Divergences in Machine Learning.}
The $f$-divergence family~\cite{Csiszar1967,AliSilvey1966} provides a unified framework for measuring distributional discrepancy. The Csisz\'ar--Amari $\alpha$-divergence~\cite{Csiszar1967,Amari2016} continuously connects forward and reverse KL. Prior work has explored $f$-divergences in variational inference~\cite{Li2016Renyi}, GANs~\cite{Nowozin2016fGAN}, and imitation learning~\cite{Ghasemipour2020fIL}. In RL, $\alpha$PPO~\cite{Xu2023AlphaPPO} studied $\alpha$-divergence as a trust-region constraint. Recently, APO~\cite{zixian2025apo} explored combining forward and reverse KL dynamics for standard preference optimization. OPO builds on these foundations by decomposing the divergence into independent geometry axes.

\paragraph{Trust-Region Methods.}
TRPO~\cite{Schulman2015TRPO} enforces stability via explicit KL constraints, while PPO~\cite{Schulman2017PPO} approximates this with ratio clipping. ADPO~\cite{Wang2025ADPO} shows that anchored coordinates provide an implicit trust region via temperature-scaled curvature. OPO extends this by replacing the KL-based geometry with a quadratic ($\chi^2$) geometry in ratio coordinates, with the geometric justification provided by the Hilbert Projection Theorem.

\section{Preliminaries: Ratio Coordinates and Conservation}
\label{sec:prelim}

Let $\pik$ denote the anchor (reference) policy at iteration $k$, with $\pik(y)>0$ on its support.

\begin{definition}[Ratio and Log-Ratio Coordinates]
\label{def:ratio_coords}
Define
\begin{equation}
t(y) := \frac{\pi(y)}{\pik(y)}, \qquad
v(y) := t(y)-1 = \frac{\pi(y)}{\pik(y)}-1,
\qquad
\Delta(y) := \log\pi(y)-\log\pik(y).
\end{equation}
Then $t(y)=e^{\Delta(y)}$ and $v(y)=e^{\Delta(y)}-1$.
\end{definition}

\paragraph{Conservation law.}
Probability normalization $\sum_y \pi(y)=1$ is equivalent to
\begin{equation}
\label{eq:conservation}
\E_{\pik}[t]=1 \quad\Longleftrightarrow\quad \E_{\pik}[v]=0.
\end{equation}
Hence admissible fluctuations lie in the closed subspace
\begin{equation}
\label{eq:H0}
\mathcal{H}_0 := \{ v \in L^2(\pik) : \E_{\pik}[v]=0\}.
\end{equation}

\begin{remark}[Geometric Interpretation of $\mathcal{H}_0$]
$\mathcal{H}_0$ is precisely the orthogonal complement of $\mathrm{span}\{\mathbf{1}\}$ in $L^2(\pik)$: every valid fluctuation must be orthogonal to the constant function. Probability mass gained at some outputs must be exactly compensated elsewhere. The term ``orthogonalized'' refers directly to this geometric fact.
\end{remark}

\section{OPO as Orthogonal Projection in Hilbert Space}
\label{sec:hilbert_framework}

This section provides the primary theoretical derivation of OPO. We formulate alignment as a constrained proximal optimization in $\mathcal{H} = L^2(\pik)$ and solve it via the Hilbert Projection Theorem.

\subsection{The $\alpha$-Escort Sampling Field}

Let $P^*$ be an oracle target distribution encoding preference/quality.
A canonical choice is $P^*(y)\propto \exp(A(y))$ where $A$ is an advantage-like signal.

\begin{definition}[$\alpha$-Escort Distribution and Escort Weight]
\label{def:target}
For $\alpha\in[0,1]$, define the escort distribution
\begin{equation}
\label{eq:escort_def}
\rho_\alpha(y) \propto \pik(y)^\alpha\,P^*(y)^{1-\alpha}.
\end{equation}
Define the escort weight (Radon--Nikodym derivative)
\begin{equation}
\label{eq:omega_alpha}
\omega_\alpha(y) := \frac{\rho_\alpha(y)}{\pik(y)}.
\end{equation}
Since $\rho_\alpha$ is normalized,
\begin{equation}
\label{eq:omega_mean_one}
\E_{\pik}[\omega_\alpha]=1.
\end{equation}
\end{definition}

\begin{proposition}[$\alpha$-Geometric Interpolation]
\label{prop:alpha_mixture}
The family $\{\rho_\alpha\}_{\alpha \in [0,1]}$ forms an $e$-geodesic connecting $P^*$ (at $\alpha=0$) to $\pik$ (at $\alpha=1$).
The escort weight simplifies to $\omega_\alpha(y) \propto \bigl(\exp(A(y))/\pik(y)\bigr)^{1-\alpha}$.
\end{proposition}

\paragraph{Geometric semantics.}
$\alpha$ controls how aggressively sampling emphasizes high-quality regions:
$\alpha\to 1$ gives $\omega_\alpha\to 1$ (conservative);
$\alpha\to 0$ gives $\omega_\alpha \to P^*/\pik$ (aggressive).

\paragraph{On-policy implementation (stop-gradient).}
When $\pik=\pi_{\mathrm{old}}$ is frozen during an iteration,
\begin{equation}
\label{eq:omega_approx}
\omega_\alpha(y)
\;\propto\;
\exp\!\big((1-\alpha)A(y)\big)\cdot
\sg\!\Big(\pik(y)^{-(1-\alpha)}\Big).
\end{equation}

\paragraph{Batch normalization of escort weights.}
The escort density $\rho_\alpha(y)\propto \pik(y)^\alpha P^*(y)^{1-\alpha}$ involves an unknown normalizing constant $Z_\alpha$.
In practice we apply batch-wise zero-centering and $z$-scoring of $\omega_\alpha$ over each mini-batch, equivalent to estimating the normalizer locally within the sample batch.

\subsection{Metric Modulation as a Multiplication Operator}

The escort weight induces a natural geometric structure in $\mathcal{H}$.

\begin{lemma}[Boundedness, Self-Adjointness, Positivity]
\label{lem:M_props}
Define the multiplication operator $(\mathcal{M}_\alpha f)(y):=\omega_\alpha(y)\,f(y)$.
If $\omega_\alpha\in L^\infty(\pik)$ and $\omega_\alpha>0$ a.s., then $\mathcal{M}_\alpha$ is bounded, self-adjoint, and positive.
The induced inner product $\langle f,g\rangle_\alpha := \E_{\rho_\alpha}[fg]$ is the escort-modulated geometry.
\end{lemma}

Thus $\alpha$-escort is a smooth \emph{metric modulation}: it preserves support and reweights directions in function space, in contrast to hard masking which truncates support.

\subsection{The Work-Dissipation Functional}

We construct a generalized free-energy functional that balances external work (alignment with the advantage signal) against intrinsic geometric dissipation (deviation from the reference):
\begin{equation}
\label{eq:opo_functional}
\mathcal{J}(v) = \underbrace{\langle \omega_\alpha, v\rangle_{\pik}}_{\text{External Work}} - \underbrace{\frac{\mu}{2}\|v\|_{L^2(\pik)}^2}_{\text{Quadratic Dissipation}}
\end{equation}
where $\mu > 0$ is the stiffness parameter. The alignment problem is:
\begin{equation}
\label{eq:opo_ratio}
\max_{v \in \mathcal{H}_0} \; \mathcal{J}(v)
\end{equation}
The constraint $v \in \mathcal{H}_0$ enforces probability conservation. The dissipation term is exactly the Pearson $\chi^2$ divergence:
$\E_{\pik}[v^2] = \E_{\pik}[(\pi/\pik-1)^2]$.

\paragraph{Orthogonality of the two geometries.}
$\alpha$ enters only through $\omega_\alpha$ (first-order drive);
$\mu$ enters only through the quadratic curvature (second-order geometry).
Sampling and optimization are \emph{structurally orthogonal}.

\subsection{Closed-Form Solution via the Hilbert Projection Theorem}
\label{subsec:projection}

We now derive the closed-form solution as a direct application of the Hilbert Projection Theorem.

\begin{theorem}[OPO Closed-Form via Orthogonal Projection]
\label{thm:closed_form_opo}
By completing the square, the functional~\eqref{eq:opo_functional} becomes:
\begin{equation}
\mathcal{J}(v) = -\frac{\mu}{2}\left\|v - \frac{1}{\mu}\omega_\alpha\right\|^2_{L^2(\pik)} + \frac{1}{2\mu}\|\omega_\alpha\|^2_{L^2(\pik)}
\end{equation}
Maximizing $\mathcal{J}(v)$ over $v \in \mathcal{H}_0$ is equivalent to minimizing $\|v - \omega_\alpha/\mu\|^2$.
Since $\mathcal{H}_0 = \{\mathbf{1}\}^\perp$ is a closed linear subspace, the Hilbert Projection Theorem guarantees a unique solution:
\begin{equation}
\label{eq:v_star}
\boxed{
v^*(y) = \frac{1}{\mu}\mathcal{P}_{\mathcal{H}_0}(\omega_\alpha) = \frac{1}{\mu}\Big(\omega_\alpha(y) - \E_{\pik}[\omega_\alpha]\Big) = \frac{1}{\mu}\Big(\omega_\alpha(y) - 1\Big).
}
\end{equation}
The last equality uses $\E_{\pik}[\omega_\alpha] = 1$ (Eq.~\ref{eq:omega_mean_one}).
\end{theorem}

\begin{proof}
The orthogonal projection $\mathcal{P}_{\mathcal{H}_0}$ onto $\mathcal{H}_0 = \{\mathbf{1}\}^\perp$ is:
\begin{equation}
\mathcal{P}_{\mathcal{H}_0}(f) = f - \frac{\langle f, \mathbf{1}\rangle_{\pik}}{\|\mathbf{1}\|^2_{\pik}}\mathbf{1} = f - \E_{\pik}[f]
\end{equation}
using $\|\mathbf{1}\|^2_{\pik} = \E_{\pik}[1] = 1$. Applying to $f = \omega_\alpha/\mu$ yields:
$v^* = (\omega_\alpha - \E_{\pik}[\omega_\alpha])/\mu = (\omega_\alpha - 1)/\mu$.
\end{proof}

\begin{remark}[The Chemical Potential]
\label{rem:chemical_potential}
The subtracted term $\lambda^* = \E_{\pik}[\omega_\alpha] = 1$ is the Lagrange multiplier enforcing probability conservation. In physics, this plays the role of a \emph{chemical potential}---the energy cost of adding one unit of probability to the system. In our framework, it emerges organically as the component of the target vector projected out by $\mathcal{P}_{\mathcal{H}_0}$, rather than being imposed ad hoc.
\end{remark}

\begin{remark}[$z$-Score Normalization as Conservation-Law Projection]
\label{rem:zscore}
The centering $\omega_\alpha \mapsto \omega_\alpha - \E[\omega_\alpha]$ is exactly $\mathcal{P}_{\mathcal{H}_0}$---the mean-subtraction step of $z$-score normalization.
When $\omega_\alpha(y) = A(y)$ ($\alpha \to 1$), the solution $v^* = (A - \E[A])/\mu$ is the $z$-scored advantage scaled by susceptibility.
Thus, advantage $z$-scoring is not a variance-reduction heuristic: it is the \emph{unique orthogonal projection} enforcing probability conservation in $L^2(\pik)$.
\end{remark}

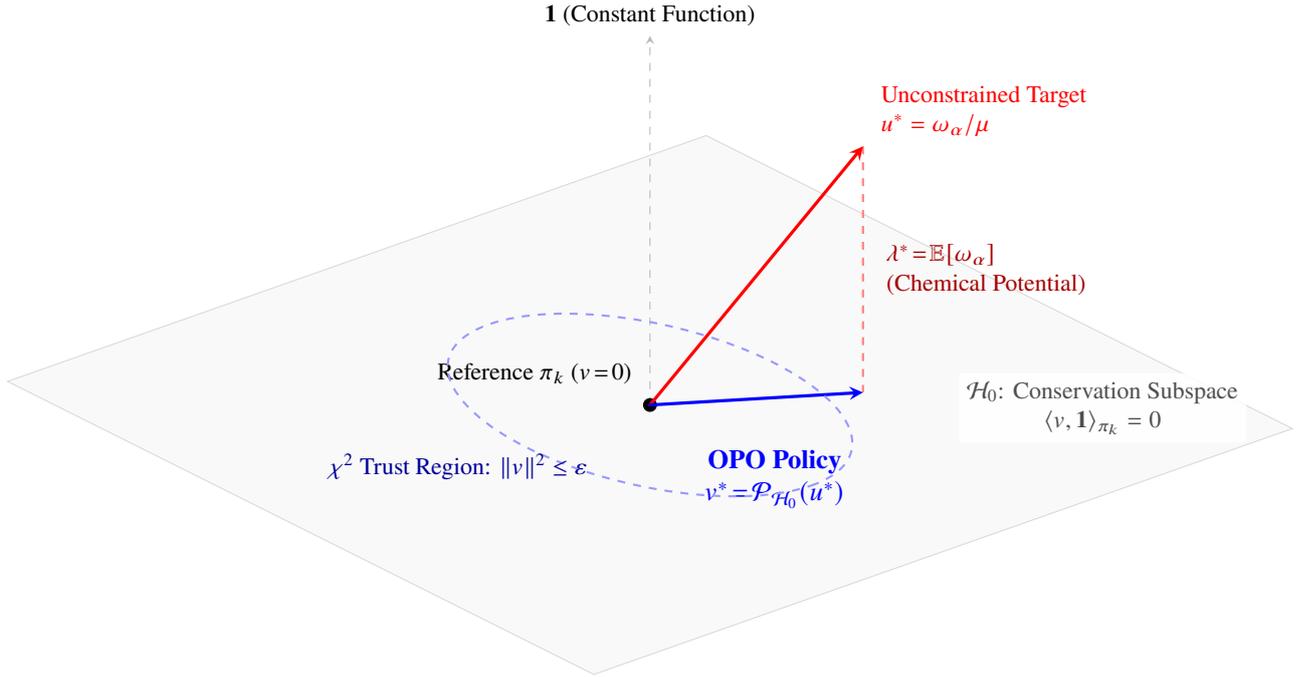
\begin{figure}[htbp]
\centering
\tdplotsetmaincoords{65}{50}

\begin{tikzpicture}[
    tdplot_main_coords,
    scale=3.0,
    >=stealth,
    label_font/.style={font=\footnotesize},
    important_label/.style={font=\small\bfseries},
]

    \coordinate (O) at (0,0,0);

    \fill[gray!10, opacity=0.45] (-2.0,-2.0,0) -- (2.0,-2.0,0) -- (2.0,2.0,0) -- (-2.0,2.0,0) -- cycle;
    \draw[gray!30, thin] (-2.0,-2.0,0) -- (2.0,-2.0,0) -- (2.0,2.0,0) -- (-2.0,2.0,0) -- cycle;

    \draw[->, gray!55, dashed, thin] (O) -- (0,0,1.8)
        node[above, black, label_font] {$\mathbf{1}$ (Constant Function)};
    \filldraw[black] (O) circle (0.8pt);

    \coordinate (U_star) at (0.5, 0.8, 1.2);
    \coordinate (V_star) at (0.5, 0.8, 0);

    \draw[dashed, red!50, line width=0.8pt] (U_star) -- (V_star);

    \draw[->, very thick, red, line cap=round] (O) -- (U_star);
    \draw[->, very thick, blue, line cap=round] (O) -- (V_star);

    \node[label_font, red, align=left, anchor=south west, xshift=3pt] at (U_star) {
        Unconstrained Target\\$u^* = \omega_\alpha/\mu$
    };

    \node[label_font, red!65!black, anchor=west, xshift=5pt, align=left] at ($(U_star)!0.5!(V_star)$) {
        $\lambda^*\!=\!\mathbb{E}[\omega_\alpha]$\\
        (Chemical Potential)
    };

    \node[important_label, blue, align=center, anchor=north east, xshift=-3pt, yshift=-18pt] at (V_star) {
        OPO Policy\\$v^*\!=\!\mathcal{P}_{\mathcal{H}_0}(u^*)$
    };

    \node[label_font, anchor=south east, xshift=-3pt, yshift=4pt] at (O) {Reference $\pi_k$ ($v\!=\!0$)};

    \draw[blue!40, dashed, thick] (O) circle [x radius=1.1, y radius=0.7];
    \node[label_font, blue!60!black, anchor=north, yshift=-3pt] at (-0.3, -0.85, 0) {
        $\chi^2$ Trust Region: $\|v\|^2 \le \varepsilon$
    };

    \node[label_font, gray!60!black, align=center,
          fill=white, fill opacity=0.85, text opacity=1,
          inner sep=3pt, rounded corners=1pt] at (1.3, 1.5, 0) {
        $\mathcal{H}_0$: Conservation Subspace\\
        $\langle v, \mathbf{1} \rangle_{\pi_k} = 0$
    };

\end{tikzpicture}

\caption{\textbf{Geometric interpretation of OPO.} The theory operates in $L^2(\pi_k)$. The reference policy $\pi_k$ sits at the origin ($v=0$). Valid policy updates must reside in $\mathcal{H}_0$ (gray plane), the codimension-one subspace orthogonal to the constant function $\mathbf{1}$. The unconstrained target $u^* = \omega_\alpha/\mu$ (red) is projected vertically onto $\mathcal{H}_0$ by subtracting the chemical potential $\lambda^* = \mathbb{E}[\omega_\alpha]$, yielding the unique optimal update $v^*$ (blue). The dashed ellipse shows the $\chi^2$ trust region; the quadratic penalty $\frac{\mu}{2}\|v\|^2$ is the Lagrangian dual of this constraint (Proposition~\ref{prop:chi2_duality}).}
\label{fig:opo_geometry}
\end{figure}

\paragraph{Policy update.}
Since $\pi=\pik(1+v)$,
\begin{equation}
\label{eq:policy_update}
\frac{\pi_{k+1}(y)}{\pik(y)}
=
1+\frac{1}{\mu}\Big(\omega_\alpha(y)-1\Big).
\end{equation}

\paragraph{Loss-minimization form.}
Equivalently,
\begin{equation}
\label{eq:opo_loss}
\boxed{
\loss_{\mathrm{OPO}}(v)
=
-\E_{\pik}[\omega_\alpha(y)\,v(y)]
+
\frac{\mu}{2}\E_{\pik}[v(y)^2],
\qquad \E_{\pik}[v]=0.
}
\end{equation}

\begin{figure}[htbp]
\centering
\includegraphics[width=0.88\textwidth]{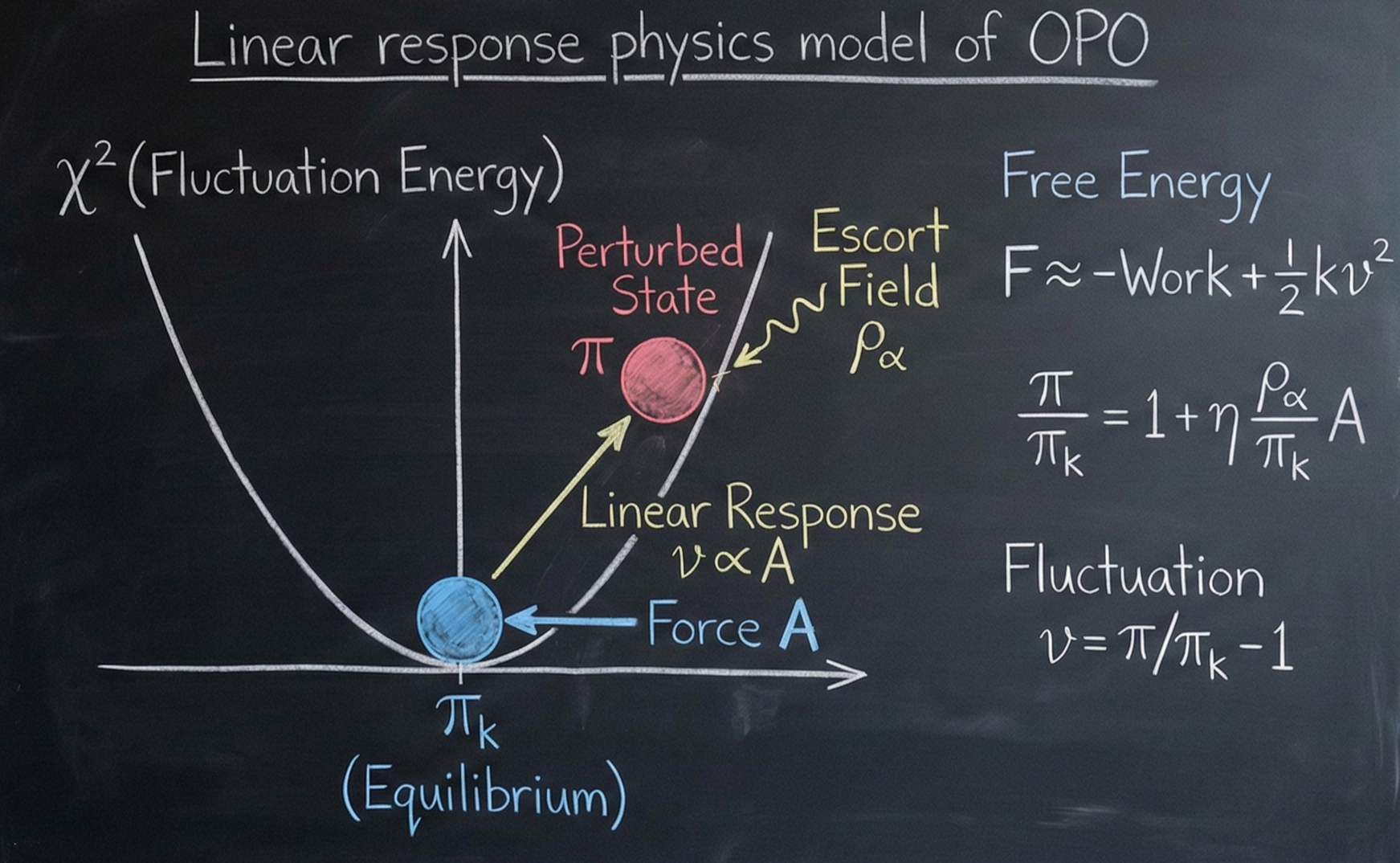}
\caption{
Hilbert-space geometry of OPO.
The anchor policy $\pik$ sits at the origin of $L^2(\pik)$.
The escort field $\rho_\alpha$ induces the target vector $\omega_\alpha/\mu$ in the full Hilbert space $\mathcal{H}$.
The conservation law $\E_{\pik}[v]=0$ defines the codimension-one subspace $\mathcal{H}_0 = \{\mathbf{1}\}^\perp$.
The Hilbert Projection Theorem yields $v^* = \mathcal{P}_{\mathcal{H}_0}(\omega_\alpha/\mu) = (\omega_\alpha - 1)/\mu$ as the unique closest point in $\mathcal{H}_0$.
The subtracted mean ($\lambda^* = 1$) is the chemical potential enforcing conservation.
}
\label{fig:opo_physics}
\end{figure}

\begin{remark}[Contrast with Fisher--Rao (KL) Geometry]
\label{rem:fisher_contrast}
In KL-based methods, the natural metric is the Fisher--Rao metric $g_{ij}^{\mathrm{FR}} = \E_\pi[\partial_i \log \pi \cdot \partial_j \log \pi]$, whose curvature depends on $\pi$. In contrast, $\chi^2$ geometry operates in $L^2(\pik)$ with \emph{constant} inner product:
\begin{center}
\begin{tabular}{lcc}
\toprule
 & \textbf{KL / Fisher--Rao} & \textbf{$\chi^2$ / $L^2(\pik)$} \\
\midrule
Inner product & Fisher metric (policy-dependent) & $L^2(\pik)$ (fixed at anchor) \\
Hessian & $\nabla^2 \propto \beta^2 \sigma(1{-}\sigma)$ & $\nabla^2 = \mu I$ (constant) \\
Perturbation theory & Nonlinear (exponential) & \textbf{Linear} (first-order exact) \\
\bottomrule
\end{tabular}
\end{center}
\end{remark}

\section{Equivalent Interpretations}
\label{sec:equivalences}

The Hilbert projection derivation is our primary framework. Here we show that the \emph{same} closed-form update arises independently from two other mathematical structures, confirming uniqueness.

\subsection{Mirror Descent Interpretation}
\label{subsec:mirror}

Define the Euclidean mirror map $\Psi(v):=\frac{1}{2}\|v\|_{L^2(\pik)}^2$. The induced Bregman divergence is:
\begin{equation}
\label{eq:bregman_chi2}
D_\Psi(\pi\|\pik) = \frac{1}{2}\E_{\pik}\!\left[\left(\frac{\pi}{\pik}-1\right)^2\right],
\end{equation}
which is the Pearson $\chi^2$ divergence. A mirror/proximal step driven by $\omega_\alpha$ is:
\begin{equation}
v^*=\argmin_{v\in \mathcal H_0}
\left\{
-\langle \omega_\alpha, v\rangle + \frac{\mu}{2}\|v\|^2
\right\},
\end{equation}
which is identical to the work--dissipation formulation~\eqref{eq:opo_ratio}.
Thus mirror descent with a Euclidean mirror map and the Hilbert projection define the \emph{same} update. Sampling geometry ($\alpha$) affects only the drive $\omega_\alpha$; optimization geometry ($\mu$) affects only the proximal curvature.

\subsection{Linear Response Interpretation}
\label{subsec:linear_response}

From near-equilibrium statistical mechanics~\cite{Kubo1966,LandauLifshitz1980}, a system at equilibrium $\pik$ subject to a weak external perturbation $\omega_\alpha$ responds linearly:
\begin{equation}
v^*(y) = \chi \cdot \big(\omega_\alpha(y) - \E[\omega_\alpha]\big)
\end{equation}
where $\chi = 1/\mu$ is the \emph{susceptibility}---the response per unit driving force. The centering enforces the conservation constraint (probability normalization as a ``thermodynamic'' identity). This is precisely the fluctuation--dissipation theorem: the response equals susceptibility times centered driving force.

\subsection{Uniqueness}

\begin{corollary}[Uniqueness of Quadratic Proximal Response]
The Hilbert projection, Euclidean mirror descent, and linear response all yield $v^* = (\omega_\alpha - 1)/\mu$. This confirms that OPO is the \emph{unique} quadratic proximal response under ratio geometry with conservation constraint.
\end{corollary}

\section{Implementable Surrogate and Algorithm}
\label{sec:implementation}

\subsection{Log-Ratio Surrogate for LLMs}

LLMs output log-probabilities. Let
$\Delta_\theta(y) := \log \pi_\theta(y) - \log \pik(y)$.
Since $v=e^{\Delta}-1\approx \Delta$ for small trust regions,
the practical surrogate is
\begin{equation}
\label{eq:opo_log_approx}
\boxed{
\loss_{\mathrm{OPO}}^{\log}(\theta)
=
-\E_{\pik}\big[\omega_\alpha(y)\,\Delta_\theta(y)\big]
+
\frac{\mu}{2}\E_{\pik}\big[\Delta_\theta(y)^2\big].
}
\end{equation}
Anchoring $\pik=\pi_{\mathrm{old}}$ resets $\Delta$ near zero each iteration, supporting the local approximation.

\subsection{Algorithm and Flowchart}

\begin{figure}[htbp]
\centering
\begin{tikzpicture}[
    node distance=1.5cm,
    auto,
    block/.style={
      rectangle,
      draw,
      fill=blue!10,
      text width=6.5em,
      text centered,
      rounded corners,
      minimum height=3em
    },
    line/.style={draw, -latex', thick},
]
    \node [block] (policy) {Current Policy $\pi_\theta$};
    \node [block, right of=policy, node distance=3.5cm] (sampling) {Sampling $(x, y)$};
    \node [block, right of=sampling, node distance=3.5cm] (eval) {Evaluation $A(y)$};
    
    \node [block, below of=sampling, node distance=2.5cm, fill=green!10, text width=13em] (opo) {\textbf{Hilbert Projection}\\ 1. Ratio $v = \pi/\pi_{old} - 1$ \\ 2. $\alpha$-Weight $\omega_\alpha = (e^A/\tilde{p})^{1-\alpha}$ \\ 3. Project: $\frac{\mu}{2}\, v^2$ penalty};
    
    \node [block, below of=opo, node distance=2.5cm] (update) {Update $\theta$};
    
    \path [line] (policy) -- (sampling);
    \path [line] (sampling) -- (eval);
    \path [line] (eval) |- (opo);
    \path [line] (sampling) -- (opo); 
    \path [line] (opo) -- (update);
    \path [line] (update) -| node [near start] {Next Iteration} (policy);
    
\end{tikzpicture}
\caption{Flowchart of the OPO framework. The ``Hilbert Projection'' block (green) embodies the work--dissipation structure: the $\alpha$-escort (step~2) shapes the external work, while the $\chi^2$ penalty (step~3) provides constant-curvature dissipation. The two axes are structurally orthogonal.}
\label{fig:flowchart}
\end{figure}
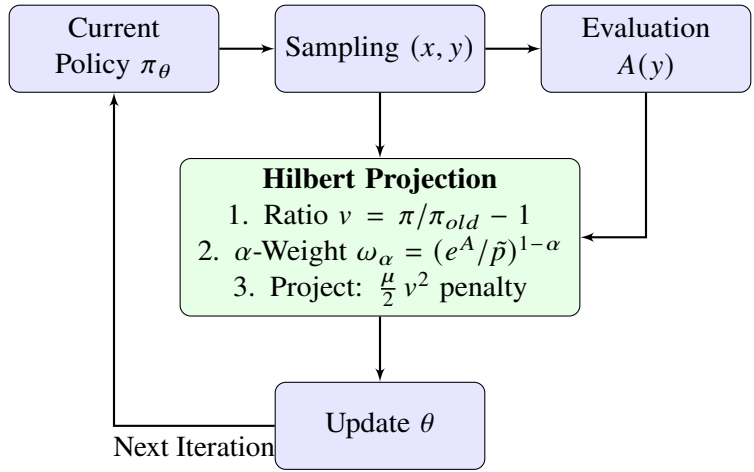

\begin{algorithm}[htbp]
\caption{Orthogonalized Policy Optimization (OPO)}
\label{alg:opo}
\begin{algorithmic}[1]
\Require Initial policy $\pi_\theta$; parameters $\alpha \in [0,1],\; \mu > 0,\; \eta > 0$
\For{each iteration}
    \State \textbf{Anchor:} $\piref \leftarrow \pi_\theta$ \Comment{On-policy anchoring}
    \State \textbf{Rollout:} Sample $(x, y) \sim \pi_\theta$; compute rewards $R(x, y)$
    \State \textbf{Advantage:} $A(y) \leftarrow \mathrm{zscore}(R(x, y))$ \Comment{$\mathcal{P}_{\mathcal{H}_0}$ projection (Remark~\ref{rem:zscore})}
    \State \textbf{Escort Weight:} $\omega_\alpha(y) \leftarrow \exp\!\big((1{-}\alpha)\, A(y)\big) \cdot \operatorname{sg}\!\big(\piref(y)^{-(1-\alpha)}\big)$ \Comment{Eq.~\ref{eq:omega_approx}}
    \State \textbf{Log-Ratio:} $\Delta_\theta(y) \leftarrow \log \pi_\theta(y) - \log \piref(y)$
    \State \textbf{Loss:} $\mathcal{L} \leftarrow \frac{1}{N} \sum_y \big[ -\omega_\alpha(y)\, \Delta_\theta(y) + \frac{\mu}{2}\, \Delta_\theta(y)^2 \big]$
    \State \textbf{Update:} $\theta \leftarrow \theta - \eta \nabla_\theta \mathcal{L}$
\EndFor
\end{algorithmic}
\end{algorithm}

\section{Theoretical Analysis}
\label{sec:analysis}

\subsection{Constant Curvature and Orthogonal Control}

\begin{theorem}[Decoupling of Sampling and Optimization Geometry]
\label{thm:decoupling}
Consider the OPO objective in $v$-space.
Then:
\begin{enumerate}[nosep]
\item The first-order drive depends on $\alpha$ only through $\omega_\alpha$:
$\nabla_v \loss_{\mathrm{OPO}} = -\omega_\alpha + \mu v$.
\item The second-order curvature is constant and independent of $\alpha$:
$\nabla_v^2 \loss_{\mathrm{OPO}} = \mu I$.
\end{enumerate}
\end{theorem}

Changing $\alpha$ reshapes the target (what to learn) without altering curvature (how to learn).
Changing $\mu$ alters curvature without reshaping the target.

\subsection{Global Contraction in $v$-Space}

Gradient descent $v^{(m+1)} = v^{(m)} - \eta(\mu v^{(m)} - \omega_\alpha)$ satisfies
\begin{equation}
v^{(m+1)} - v^* = (1-\eta\mu)\,(v^{(m)}-v^*).
\end{equation}

\begin{corollary}[Linear Rate]
\label{cor:linear_rate}
For $0<\eta<2/\mu$,
$\|v^{(m)}-v^*\| \le |1-\eta\mu|^m \,\|v^{(0)}-v^*\|$.
The contraction rate depends only on $\mu$ and is independent of $\alpha$.
\end{corollary}

\subsection{Why KL-Based Objectives Saturate}

In KL-regularized preference optimization, local curvature scales like $\sigma(m)(1-\sigma(m))$ and vanishes exponentially as $|m|\to\infty$.
By contrast, OPO has constant curvature $\mu I$ and linear restoring force $\|\nabla_v \loss\| = \|\mu(v-v^*)\|$, which remains proportional to distance-to-equilibrium.

Figure~\ref{fig:toy_saturation} visualizes this contrast on a binary logit toy model ($z_{\mathrm{ref}}=0$). Panel~(a) plots $|\partial\mathcal{L}/\partial z|$ versus the model confidence $p=\sigma(z)$: SFT, DPO, GRPO, DAPO, and GSPO all decay to zero in the high-confidence regime, while OPO's $\chi^2$ gradient grows linearly. Panel~(b) zooms into $p\ge0.5$ on a logarithmic scale, revealing that clipped objectives (GRPO, DAPO, GSPO) suffer a hard cliff to machine-epsilon once the ratio exits the clip window, and DPO saturates exponentially. OPO is the only objective whose gradient magnitude remains $O(1)$ throughout.

\begin{figure}[htbp]
\centering
\includegraphics[width=0.95\textwidth]{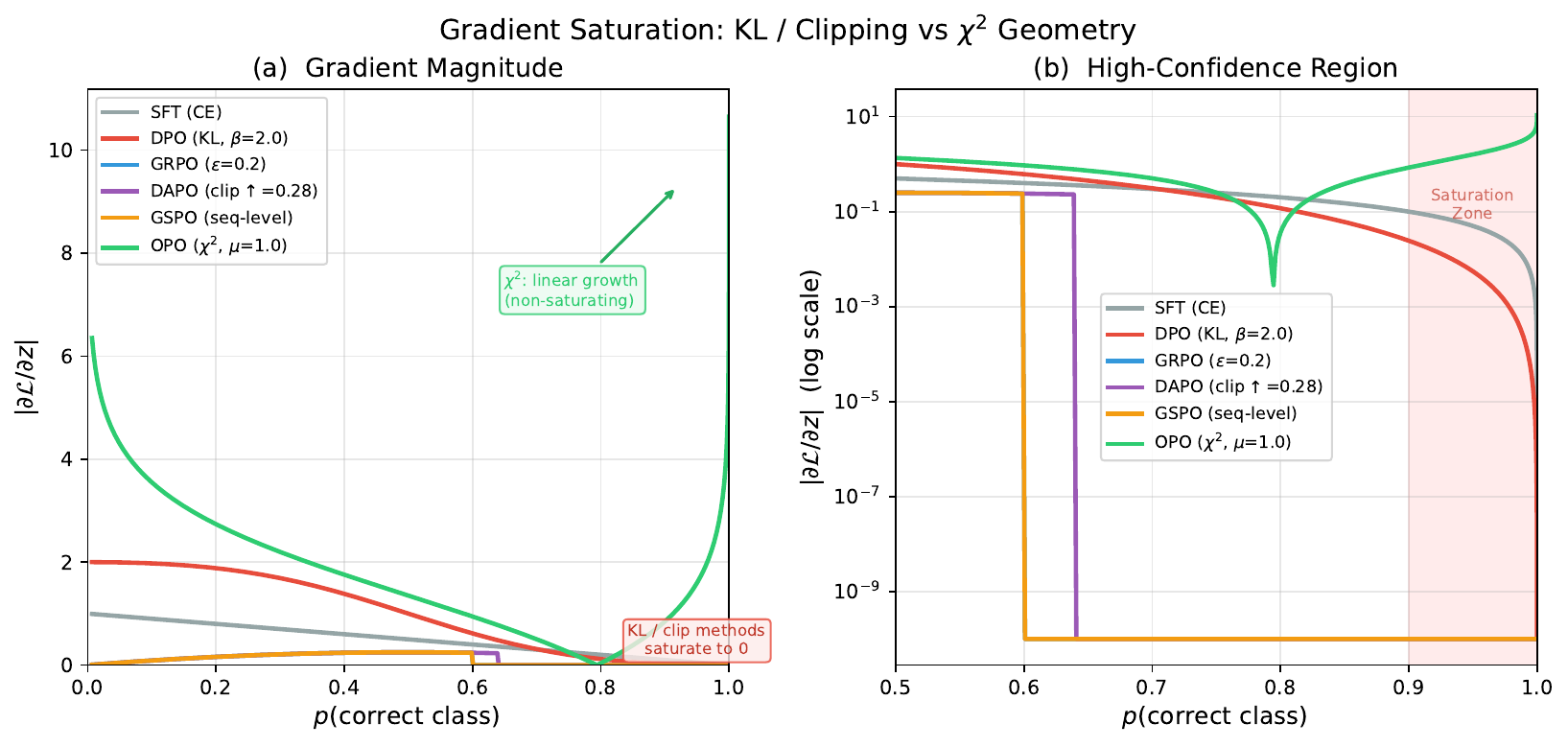}
\caption{Toy gradient saturation experiment (binary logit, $z_{\mathrm{ref}}=0$). (a)~Gradient magnitude vs.\ model confidence: OPO ($\chi^2$, green) exhibits linear growth while all other objectives saturate. (b)~Log-scale zoom on $p\ge 0.5$: clipped objectives cliff-drop by up to 9 orders of magnitude; OPO maintains non-vanishing gradients throughout.}
\label{fig:toy_saturation}
\end{figure}

\subsection{Trust-Region Duality}

\begin{proposition}[$\chi^2$ Trust-Region Duality]
\label{prop:chi2_duality}
The constrained problem
$\max_{v\in\mathcal H_0} \langle \omega_\alpha, v\rangle$ s.t. $\|v\|^2 \le \varepsilon$
has Lagrangian relaxation $\max_{v\in\mathcal H_0} \langle \omega_\alpha, v\rangle - \frac{\mu}{2}\|v\|^2$,
with $\mu$ as the multiplier.
Moreover, $\|v\|^2 = \E_{\pik}[(\pi/\pik-1)^2]$ is the Pearson $\chi^2$ divergence.
\end{proposition}

\subsection{Positivity Safeguards}
\label{sec:positivity}

The closed-form linear response $v^*(y)=(\omega_\alpha(y)-1)/\mu$ is derived in the tangent space of the probability manifold at $\pik$, and only guarantees first-order stationarity of the proximal objective. It does \emph{not} enforce $\pik(y)(1+v^*(y))\ge 0$ outside a local neighborhood.

\begin{proposition}[Local Positivity Condition]
\label{prop:positivity}
The update $\pi_{k+1}=\pik(1+v^*)$ is a valid probability distribution whenever $\mu \ge \|\omega_\alpha\|_\infty - 1$, since then $v^*(y)=(\omega_\alpha(y)-1)/\mu \ge -1$ for all $y$.
\end{proposition}

In practice, positivity is maintained by one or more of: (i)~choosing $\mu$ large enough relative to the escort weight range; (ii)~clipping $v^*$ to $[-1+\epsilon,\,+\infty)$; (iii)~projecting onto the probability simplex after the update.

\section{Experiments}
\label{sec:experiments}

We conduct two sets of experiments to validate OPO's theoretical predictions. Experiment~1 uses the full MATH Level 3 training set to evaluate long-horizon training dynamics and gradient saturation behavior. Experiment~2 uses a deliberately small training subset with out-of-distribution validation to test generalization under data scarcity.

\subsection{Experiment 1: Training Dynamics and Gradient Saturation}

\paragraph{Setup.} We use Qwen3-1.7B as the base model and train on MATH Level 3 problems using the VERL framework. All methods use identical training configurations: 4 epochs, batch size 32, learning rate $2 \times 10^{-6}$, and 6 rollout generations per prompt.

\paragraph{Baseline Configurations.}
\begin{itemize}
    \item \textbf{GRPO} (Group Relative Policy Optimization)~\cite{Shao2024GRPO}: Standard token-level policy gradient with group-normalized advantages. Uses vanilla PPO loss mode without a value function critic.
    \item \textbf{GSPO} (Group Sequence Policy Optimization)~\cite{Zheng2025GSPO}: Applies sequence-level importance ratios rather than token-level ratios, aiming for more granular credit assignment at the trajectory level.
    \item \textbf{DAPO} (Decoupled Clip and Dynamic sAmpling Policy Optimization)~\cite{Yu2025DAPO}: An advanced baseline incorporating four key mechanisms:
    \begin{enumerate}
        \item \textbf{Clip Higher}: Raises the PPO clip upper bound (e.g., $1+\epsilon \to 1.28$) to mitigate entropy collapse and preserve exploration.
        \item \textbf{Dynamic Sampling}: dynamically filters ``all-correct'' or ``all-wrong'' samples that contribute zero advantage, focusing training on informative samples.
        \item \textbf{Token-Level Policy Gradient}: Standardizes advantage weights across tokens in a mini-batch to prevent gradient dilution in long sequences.
        \item \textbf{Overlong Reward Shaping}: Applies a length-aware penalty and masks gradients for truncated responses to reduce noise.
    \end{enumerate}
    \item \textbf{OPO}: The proposed method with $\alpha = 0.4$, $\mu = 1.0$, on-policy anchoring ($\piref = \pi_{\text{old}}$), and adaptive $\tau$ with range $[0.2, 1.5]$. Note: $1-\alpha = 0.6$ controls the effective advantage amplification in Eq.~\eqref{eq:omega_approx}.
\end{itemize}

\paragraph{Results.} Table~\ref{tab:results} summarizes the key metrics across all four algorithms.

\begin{table}[htbp]
\centering
\caption{Experiment 1: Comparison of alignment algorithms on Qwen3-1.7B + full MATH Level 3. Reward and gradient norm are averaged over the final 20 training steps. Higher gradient norm at convergence, combined with stable (non-exploding) training dynamics, indicates absence of gradient saturation rather than optimization instability.}
\label{tab:results}
\begin{tabular}{lccc}
\toprule
\textbf{Algorithm} & \textbf{Mean Reward} & \textbf{Grad Norm} & \textbf{Characteristics} \\
\midrule
GRPO & 0.686 & 0.61 & Standard baseline, stable \\
GSPO & 0.713 & 0.50 & Strong, some gradient decay \\
DAPO & 0.67* & 0.22 & Conservative, early plateau \\
\textbf{OPO (Ours)} & \textbf{0.756} & \textbf{0.90} & Best reward, non-saturating gradients \\
\bottomrule
\end{tabular}
\end{table}

\paragraph{Algorithm Characteristics.}
\begin{itemize}
    \item \textbf{GRPO}: Serves as the canonical baseline. Achieves reasonable performance (69\% acc) with moderate gradient norms. The token-level normalization provides stability but may average out strong signals.
    \item \textbf{GSPO}: Improves upon GRPO (71\% acc) via sentence-level credit assignment. However, gradient norms (0.50) are lower than OPO's, suggesting some saturation in high-confidence regimes.
    \item \textbf{DAPO}: Exhibits the lowest gradient norms (0.22) and plateaus early. The conservative clipping constraints appear overly restrictive for reasoning tasks where continued learning is beneficial. *Note: Reward fluctuated between 0.59--0.67.
    \item \textbf{OPO}: Achieves the highest mean reward (0.756) while maintaining the largest gradient norms (0.90). This empirically validates the theoretical prediction: the $\chi^2$ geometry provides non-saturating linear gradients, allowing the model to continue learning where others plateau.
\end{itemize}

\paragraph{Training Dynamics.} Figure~\ref{fig:training} shows the training accuracy (mean reward) over 224 training steps.
\begin{itemize}
    \item \textbf{Early Stage}: GRPO exhibits the fastest initial learning, surpassing OPO in the first 50 steps. This is consistent with standard policy gradient methods having aggressive early updates.
    \item \textbf{Late Stage Cross-over}: As training progresses (after step 100), GRPO and GSPO begin to saturate. In contrast, OPO maintains a steady improvement rate, eventually overtaking all baselines to reach the highest final accuracy ($\sim$76\%). This confirms that the $\chi^2$-induced linear drive prevents the ``vanishing gradient'' problem common in KL-constrained methods as confidence increases.
\end{itemize}

\begin{figure}[htbp]
\centering
\includegraphics[width=0.85\textwidth]{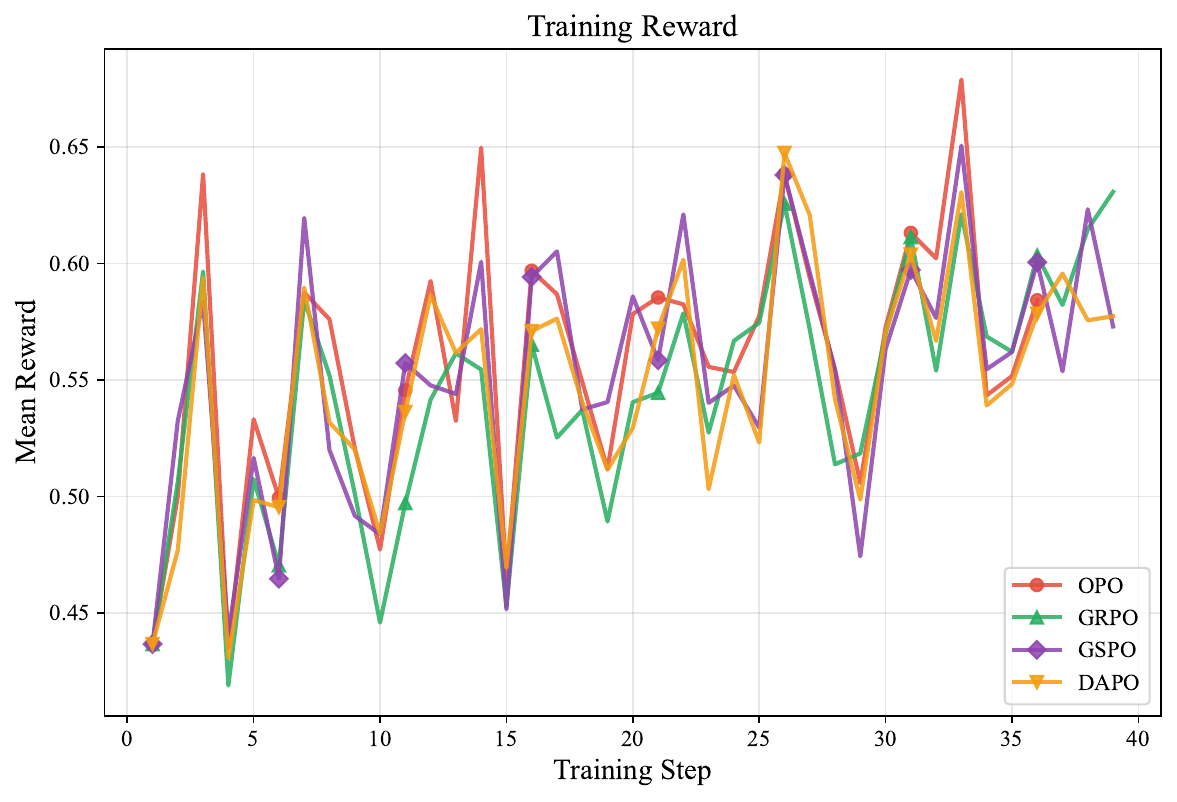}
\caption{Training dynamics on Qwen3-1.7B Math RL. Note the cross-over behavior: GRPO (green) starts strong but plateaus, while OPO (red) sustains improvement and achieves the highest final performance.}
\label{fig:training}
\end{figure}

\paragraph{Gradient Behavior.} Figure~\ref{fig:gradient} illustrates the gradient norm dynamics across training. OPO exhibits consistently higher but stable (non-exploding) gradient norms throughout training, validating the non-saturating property of the $\chi^2$ geometry. The stability of these norms rules out the alternative explanation of optimization instability or under-regularization. DAPO shows severely diminished gradients, explaining its early plateau.

\begin{figure}[htbp]
\centering
\includegraphics[width=0.85\textwidth]{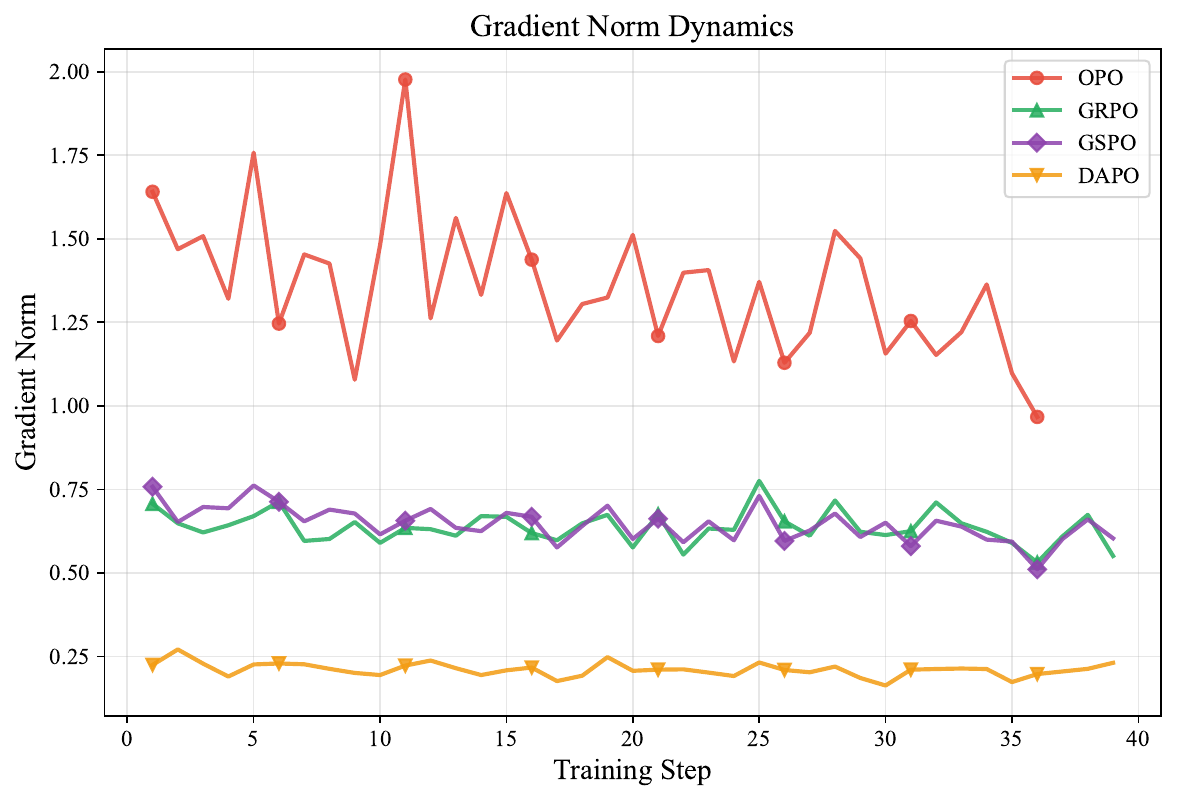}
\caption{Gradient norm comparison. OPO maintains healthy gradient norms (0.9) throughout training, while DAPO (0.22) and GSPO (0.50) exhibit gradient decay.}
\label{fig:gradient}
\end{figure}

\paragraph{Gradient Saturation Analysis.} Figure~\ref{fig:grad_saturation} provides direct evidence for OPO's non-saturating property. Panel~(a) plots gradient norms against mean reward (a confidence proxy). As the model becomes more confident (higher reward), GRPO and GSPO gradients exhibit a decaying trend, and DAPO gradients remain near zero. In contrast, OPO maintains elevated gradient norms across all confidence levels.

Panel~(b) plots the \emph{gradient response efficiency} $\|\nabla\|/\bigl(\tfrac{1}{2}\text{AdvRange}\bigr)$ over training steps, normalizing gradient magnitude by the advantage signal amplitude. OPO sustains an efficiency of ${\sim}1.0$--$1.5$ throughout training, approximately $5{\times}$ higher than GRPO (${\sim}0.2$--$0.3$) and DAPO (${\sim}0.1$). This confirms the linear response prediction of Theorem~\ref{thm:closed_form_opo}: under $\chi^2$ geometry, each unit of advantage signal produces a proportional gradient update, whereas KL-based objectives attenuate the signal through their policy-dependent curvature.

\begin{figure}[htbp]
\centering
\includegraphics[width=0.95\textwidth]{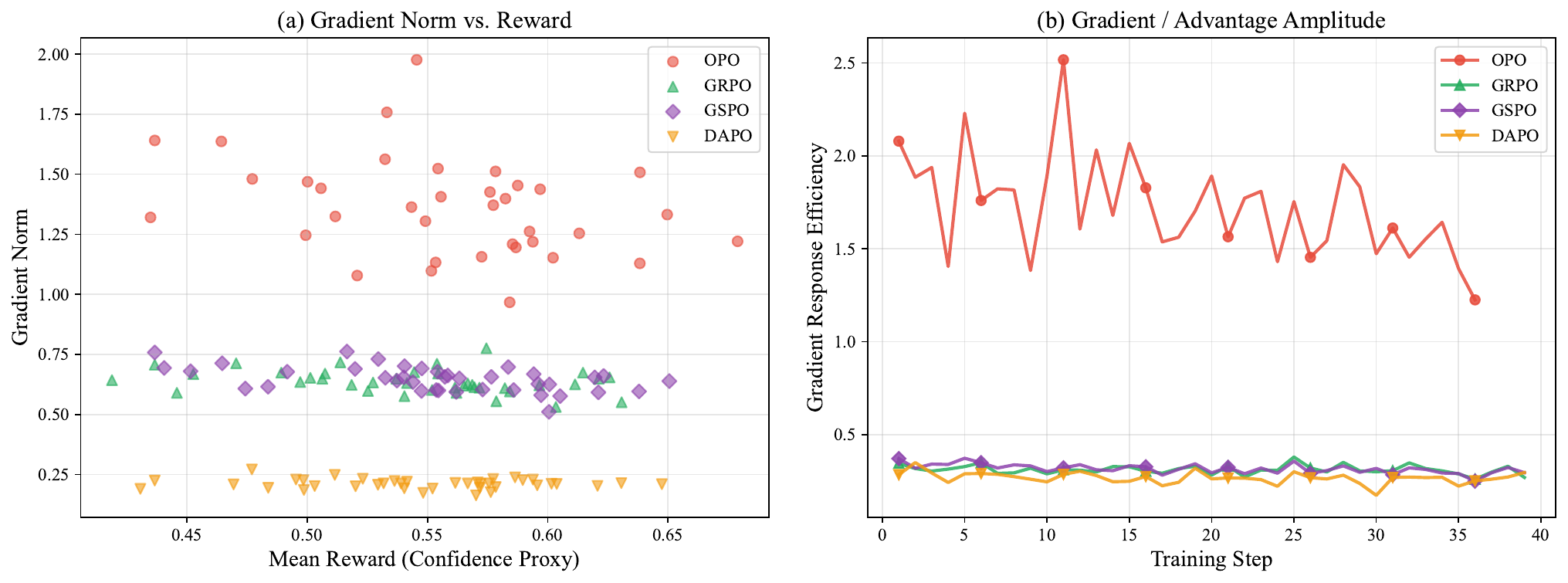}
\caption{Gradient saturation analysis. (a)~Gradient norm vs.\ mean reward (confidence proxy): OPO maintains non-saturating gradients as confidence grows, while baselines decay. (b)~Gradient response efficiency (gradient norm normalized by advantage amplitude) over training: OPO sustains ${\sim}5{\times}$ higher gradient-per-signal than baselines, confirming the constant-curvature prediction of $\chi^2$ geometry.}
\label{fig:grad_saturation}
\end{figure}

\paragraph{Entropy Dynamics.} Figure~\ref{fig:entropy} depicts the policy entropy throughout training. OPO maintains consistently higher entropy levels compared to baselines. We attribute this to the granularity of supervision: GRPO and DAPO utilize \textbf{token-level} updates, imposing dense supervision that forces the policy to collapse rapidly into specific phrasings. In contrast, OPO (and GSPO) employs \textbf{sequence-level} updates, where $y$ denotes an entire response trajectory and the loss is averaged over complete sequences rather than individual tokens. This ``sparse'' supervision aligns the total trajectory probability without micromanaging individual tokens, thereby preserving policy diversity and exploration potential.

\begin{figure}[htbp]
\centering
\includegraphics[width=0.85\textwidth]{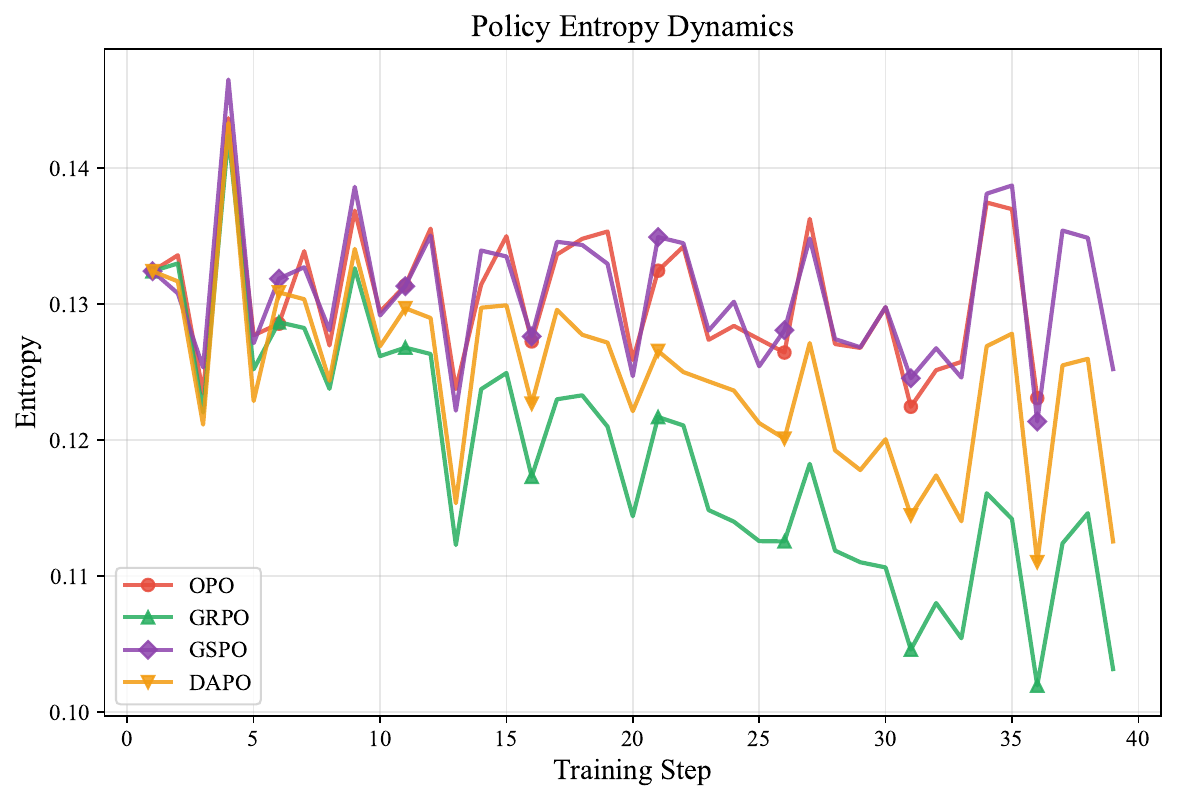}
\caption{Entropy dynamics. OPO maintains higher entropy than baselines, preventing premature mode collapse and enabling sustained exploration throughout the training process.}
\label{fig:entropy}
\end{figure}

\paragraph{Summary of Experiment 1.} The long-horizon experiment confirms that OPO's orthogonalized design---combining $\alpha$-weighted sampling with $\chi^2$ optimization geometry---prevents gradient saturation and enables continued learning in high-confidence regimes. The gradient saturation analysis (Figure~\ref{fig:grad_saturation}) provides the most direct evidence: OPO's gradient magnitude remains proportional to advantage amplitude (linear response), while KL-based baselines exhibit suppressed gradients regardless of the driving signal.

\subsection{Experiment 2: Out-of-Distribution Generalization}

To further evaluate OPO under data scarcity and test generalization to harder problems, we conduct a second experiment with a deliberately reduced training set and out-of-distribution validation.

\paragraph{Setup.} Using the same VERL framework~\cite{Sheng2024HybridFlow} on 4$\times$ RTX 4090 GPUs, we train Qwen3-1.7B on approximately 10\% of MATH Level 3 problems (sampled with seed 42). Validation is performed every 10 training steps on 100 held-out MATH Level 4 problems---a strictly harder difficulty tier not seen during training. All methods share identical hyperparameters: batch size 48, learning rate $2 \times 10^{-6}$, 8 epochs over the training data, and $G=6$ rollout generations per prompt. The same baseline configurations (GRPO, GSPO, DAPO) and OPO hyperparameters are used.

\paragraph{Results.} Table~\ref{tab:generalization} summarizes the metrics for the generalization experiment.

\begin{table}[htbp]
\centering
\caption{Experiment 2: Out-of-distribution generalization. All methods use Qwen3-1.7B trained on $\sim$10\% of MATH Level 3, validated on 100 MATH Level 4 problems. Mean Reward is averaged over the entire training process.}
\label{tab:generalization}
\begin{tabular}{lcccc}
\toprule
\textbf{Algorithm} & \textbf{Mean Reward}$\uparrow$ & \textbf{Val Acc (L4)}$\uparrow$ & \textbf{Grad Norm} & \textbf{Entropy} \\
\midrule
GRPO & 0.544 & 44\% & 0.674 & 0.115 \\
DAPO & 0.548 & 44\% & 0.213 & 0.126 \\
GSPO & 0.553 & \textbf{48\%} & 0.623 & 0.128 \\
\textbf{OPO (Ours)} & \textbf{0.558} & \textbf{48\%} & \textbf{1.279} & 0.126 \\
\bottomrule
\end{tabular}
\end{table}

\paragraph{Analysis.}
\begin{itemize}
    \item \textbf{Generalization vs.\ Training Reward.} OPO achieves the highest mean reward (0.558) over the training process, which translates to superior generalization (48\% validation accuracy on MATH Level 4, tied with GSPO). GRPO and DAPO plateau at 44\% despite comparable training rewards, suggesting that clipping-based methods tend to exploit the training distribution rather than learning generalizable reasoning patterns.
    \item \textbf{Gradient Dynamics.} Consistent with Experiment 1, OPO maintains substantially higher gradient norms (1.28) than ratio-clipped methods (GRPO: 0.67, GSPO: 0.62, DAPO: 0.21), confirming the non-saturation property under data-scarce conditions as well.
    \item \textbf{Entropy.} All sequence-level methods (OPO, GSPO) preserve higher entropy than token-level methods (GRPO), consistent with the mechanism discussed in Experiment 1.
\end{itemize}

\paragraph{Generalization Trajectories.} Figure~\ref{fig:val_accuracy} shows the validation accuracy on MATH Level 4 over training. OPO starts strong and maintains competitive generalization throughout, while GRPO and DAPO plateau at lower levels despite continued training.

\begin{figure}[htbp]
\centering
\includegraphics[width=0.85\textwidth]{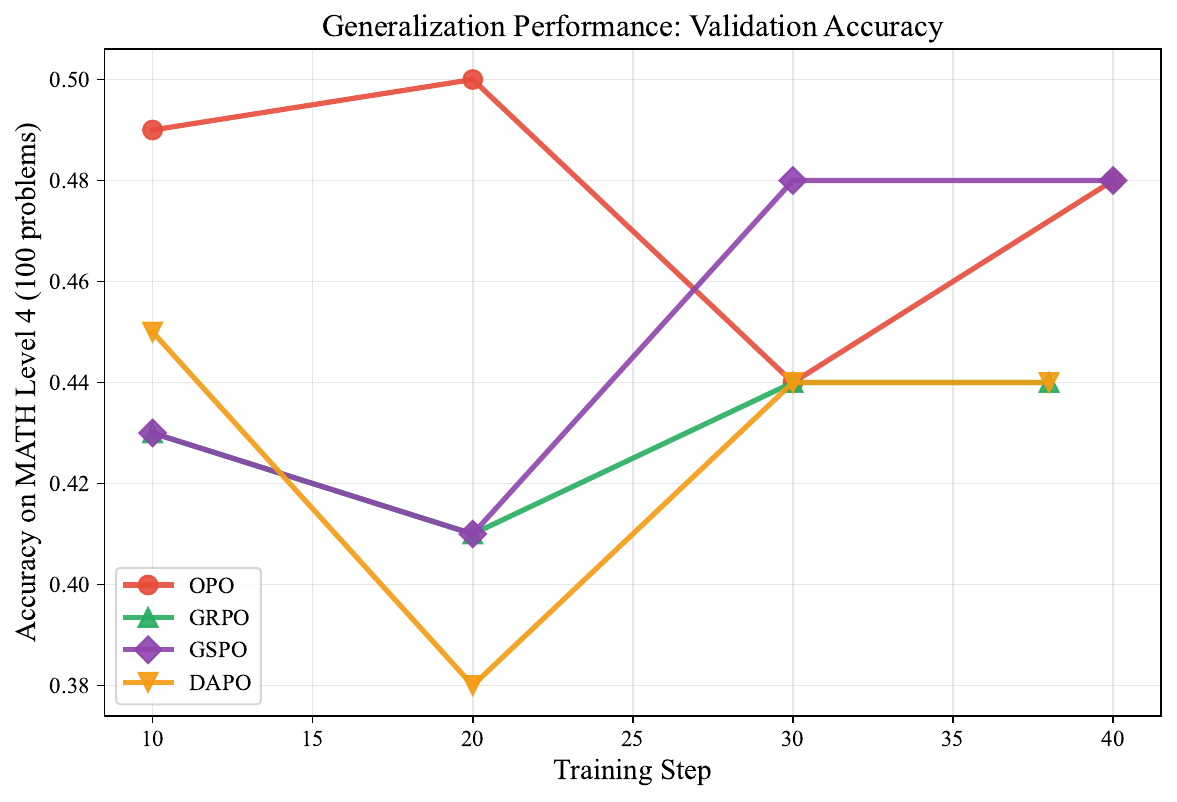}
\caption{Validation accuracy on held-out MATH Level 4 problems (out-of-distribution). OPO and GSPO achieve 48\% accuracy, outperforming GRPO and DAPO (44\%).}
\label{fig:val_accuracy}
\end{figure}

\paragraph{Summary of Experiment 2.} Under data scarcity and out-of-distribution evaluation, OPO achieves the best combination of training reward and generalization, confirming that the $\chi^2$ geometry's non-saturating gradients produce more transferable learning signals than clipping-based alternatives.

\section{Discussion}

\paragraph{Relation to Existing Methods.}
OPO can be viewed as occupying a previously unexplored point in the design space:
\begin{itemize}
    \item $\alpha \to 1$ (conservative escort): $\omega_\alpha \to 1$, recovering supervised fine-tuning--like uniform weighting.
    \item $\alpha \to 0$ (aggressive escort) with KL geometry: the exponential-type weighting resembles KL-regularized preference optimization (e.g., DPO), but inherits gradient saturation.
    \item \textbf{OPO}: moderate $\alpha$ (peak-seeking escort) \emph{combined with} Euclidean ($\chi^2$) geometry---non-saturating gradients with advantage-aware sampling.
\end{itemize}

\paragraph{Geometry Comparison.}
Table~\ref{tab:geometry} summarizes the structural differences among representative methods.

\begin{table}[htbp]
\centering
\caption{Comparison of sampling and optimization geometries across alignment methods.}
\label{tab:geometry}
\begin{tabular}{lcc}
\toprule
\textbf{Method} & \textbf{Sampling Geometry} & \textbf{Optimization Geometry} \\
\midrule
PPO / GRPO & implicit KL & KL curvature (ratio clip) \\
DPO & KL & KL curvature (sigmoid) \\
\textbf{OPO (Ours)} & $\alpha$-escort (tunable) & $\chi^2$ curvature (constant $\mu I$) \\
\bottomrule
\end{tabular}
\end{table}

\paragraph{Geometric Interpretation of Entropy Preservation.}
From a geometric perspective, KL geometry inherits a logarithmic penalty on low-probability outcomes: $-\log\pi(y)$ diverges as $\pi(y)\to 0$, strongly suppressing entropy by penalizing any residual mass on unlikely tokens. In contrast, $\chi^2$ geometry penalizes deviations quadratically in ratio space: the cost $(\pi/\pik-1)^2$ remains finite even when $\pi(y)\ll \pik(y)$, which attenuates the shrinkage of low-probability mass and thus helps maintain entropy.

\paragraph{Computational Cost.}
OPO incurs no additional model forward/backward passes compared to DPO/GRPO; the extra bookkeeping for $\omega_\alpha$ is lightweight in practice.

\paragraph{Role of the Reference Policy.}
In OPO, the reference policy $\piref$ serves solely as an \emph{origin} defining the coordinate system for the ratio $v_\theta = \pi_\theta/\piref - 1$, rather than as an explicit regularizer via KL divergence. Stability is instead enforced by the quadratic penalty $\frac{\mu}{2} v^2$. This decoupling clarifies the distinct roles of anchoring (coordinate system) and regularization (optimization geometry).

\paragraph{Implementation Notes.}
The OPO objective (\cref{eq:opo_ratio}) involves expectations under $\piref$. In practice:
\begin{itemize}
    \item For on-policy training where samples come from $\piref = \pi_{\text{old}}$, the batch itself provides an unbiased estimate.
    \item \textbf{Mini-batch anchoring.} Within one outer iteration, $\piref$ is frozen at the policy snapshot taken \emph{before} the first mini-batch gradient step. All subsequent mini-batches in the same iteration share this anchor, ensuring that the log-ratio $\Delta_\theta = \log\pi_\theta - \log\piref$ is evaluated against a consistent reference.
    \item For off-policy settings, importance sampling or a separate reference sample set can be used.
\end{itemize}

\paragraph{Limitations.}
The framework introduces hyperparameters ($\alpha$, $\mu$) that may require tuning. Our experiments focus on reasoning tasks; validation on diverse domains (code generation, general instruction following) remains future work. Larger-scale experiments across diverse tasks and model sizes are needed to fully characterize the strengths and limitations of OPO.

\section{Conclusion}

We have presented \textbf{Orthogonalized Policy Optimization (OPO)}, an alignment framework grounded in the geometry of the Hilbert function space $L^2(\pik)$. By lifting policy optimization from the probability simplex into this function space, the normalization constraint becomes a linear orthogonality condition, and the optimal policy update is obtained via the Hilbert Projection Theorem---with the chemical potential emerging organically as the projected-out component. This primary derivation is complemented by two equivalent interpretations (Euclidean mirror descent and near-equilibrium linear response), confirming that OPO is the unique quadratic proximal response in ratio geometry.

The resulting framework cleanly decouples sampling geometry ($\alpha$) from optimization geometry ($\mu$), yields constant Hessian $\mu I$, and maintains non-saturating linear gradients. The Hilbert projection perspective further reveals that advantage $z$-score normalization is not a variance-reduction heuristic but the unique conservation-law projection. Experiments on mathematical reasoning tasks validate that OPO outperforms KL-based methods where gradient saturation limits further improvement, while maintaining healthy gradient dynamics and policy diversity throughout training.

\appendix
\section{Theoretical Proofs and Derivations}
\label{app:proofs}

\subsection{Log-Ratio Approximation Error}
\begin{lemma}[Log-Ratio Approximation Error]
\label{lem:log_error_app}
Let $\Delta_\theta(y) = \log \pi_\theta(y) - \log \piref(y)$ be the log-ratio, and $v_\theta(y) = \exp(\Delta_\theta(y)) - 1$ be the exact ratio deviation. Let $\delta = \|\Delta_\theta\|_\infty$. For sufficiently small $\delta$ (in particular $\delta < 1$, which holds under the small trust-region regime maintained by on-policy anchoring and the $\chi^2$ penalty $\frac{\mu}{2}\E[v^2]$):
\begin{equation}
    |v_\theta - \Delta_\theta| \le \frac{1}{2}\delta^2 e^\delta, \quad |v_\theta^2 - \Delta_\theta^2| \le \delta^3 e^{2\delta}
\end{equation}
Consequently, the difference between the exact ratio loss and the log-approximate loss scales as $O(\E[\Delta_\theta^3])$.
\end{lemma}

\begin{proof}
Using Taylor expansion $e^x = 1 + x + \frac{x^2}{2} + \dots$, we have $v = \Delta + \frac{\Delta^2}{2} + O(\Delta^3)$. The bounds follow from standard remainder estimation for the exponential function.
\end{proof}

\subsection{Equivalence to $\chi^2$-Constrained Maximization}
\begin{proposition}[Lagrange Dual Equivalence]
Consider the problem of maximizing alignment with the target $\omega_\alpha$ subject to a functional $\chi^2$ trust region:
\begin{equation}
    \max_{v \in L^2(\piref)} \E_{\piref}[\omega_\alpha(y) v(y)] \quad \text{s.t.} \quad \E_{\piref}[v(y)^2] \le \epsilon
\end{equation}
The Lagrangian relaxation of this problem is exactly the OPO objective $\mathcal{L}_{\text{OPO}} = -\E[\omega_\alpha v] + \frac{\mu}{2} \E[v^2]$, where the regularization coefficient $\mu$ is the Lagrange multiplier corresponding to the trust region radius $\epsilon$.
\end{proposition}
The optimal solution is $v^*(y) = \frac{1}{\mu} \omega_\alpha(y)$. This duality implies that $\alpha$ purely shapes the objective (alignment direction), while $\mu$ purely enforces the feasibility radius (optimization geometry).

\subsection{Proof of Proposition~\ref{prop:alpha_mixture}: $\alpha$-Geometric Interpolation}
\label{app:alpha_proof}

\begin{proof}
Define the oracle target $P^*(y) \propto \exp(A(y))$ and the parameterized family $\rho_\alpha(y) \propto \tilde{p}(y)^\alpha\, P^*(y)^{1-\alpha}$ for $\alpha \in [0,1]$.

\textbf{Step 1 (Geodesic structure).}
In the $e$-affine coordinate system of the statistical manifold~\cite{Amari2016}, the log-density of $\rho_\alpha$ is:
\begin{equation}
\log \rho_\alpha(y) = \alpha\log \tilde{p}(y) + (1-\alpha) \log P^*(y) - \log Z_\alpha
\end{equation}
where $Z_\alpha$ is the normalizing constant. Since $\log \rho_\alpha$ is an affine function of $\alpha$ in the natural parameter, $\{\rho_\alpha\}_{\alpha \in [0,1]}$ forms an \emph{$e$-geodesic} (exponential geodesic) connecting $P^*$ (at $\alpha=0$) to $\tilde{p}$ (at $\alpha=1$) in the exponential family.

\textbf{Step 2 (Importance weight).}
The sampling weight is:
\begin{equation}
\omega_\alpha(y) = \frac{\rho_\alpha(y)}{\tilde{p}(y)} \propto \left(\frac{P^*(y)}{\tilde{p}(y)}\right)^{1-\alpha}
\end{equation}
In the on-policy setting where $P^*(y) = \exp(A(y))$ and $\tilde{p} = \pi_{\mathrm{old}}$, this yields $\omega_\alpha \propto (\exp(A)/\tilde{p})^{1-\alpha}$, which simplifies to $\exp((1-\alpha)A)$ when the $\tilde{p}$ correction is absorbed by the sampling measure.

\textbf{Step 3 (Geometric consistency).}
This interpolation corresponds to the exponential geodesic in the statistical manifold and is consistent with R\'enyi $\alpha$-divergence interpolation geometry~\cite{Amari2016}.
\end{proof}

\subsection{Distributional Stability Guarantees}
\begin{proposition}[$\chi^2$ Controls Total Variation]
Let $TV(\pi, \piref) = \frac{1}{2} \E_{\piref}[|t(y)-1|] = \frac{1}{2}\E[|v(y)|]$. By Jensen's inequality:
\begin{equation}
    TV(\pi, \piref) = \frac{1}{2} \E[|v|] \le \frac{1}{2} \sqrt{\E[v^2]}
\end{equation}
Thus, bounding the $\chi^2$ norm $\E[v^2] \le \epsilon$ guarantees that the policy remains within a $\sqrt{\epsilon}/2$-radius of the reference in Total Variation distance.
\end{proposition}

\end{document}